\documentclass[runningheads]{llncs}
\usepackage{amsmath}
\usepackage{amsfonts}
\usepackage{graphicx}
\usepackage{tabularx}
\usepackage{orcidlink}
\usepackage[misc]{ifsym}

\newcommand{\eref}[1]{Eq.~\ref{#1}}
\newcommand{\fref}[1]{Figure~\ref{#1}}
\newcommand{\tref}[1]{Table~\ref{#1}}
\newcommand{\sref}[1]{Section~\ref{#1}}

\newcommand{\thmref}[1]{Theorem~\ref{#1}}
\newcommand{\cref}[1]{Corollary~\ref{#1}}

\newcommand{\oref}[1]{Observation~\ref{#1}}

\newtheorem{observation}[theorem]{Observation}

\numberwithin{equation}{section}

\begin{document}
\title{Relations between Adjacency and Modularity Graph Partitioning}
%
%
\author{Hansi Jiang\inst{1}$^{\textrm{(\Letter)}}$\orcidlink{0000-0002-0895-3602} \and
Carl Meyer\inst{2}}
\authorrunning{H. Jiang and C. Meyer}
%
\institute{IoT Division, SAS Institute Inc., Cary, NC 27513, USA\\ 
\href{mailto:Hansi.Jiang@sas.com}{Hansi.Jiang@sas.com}
\and
Department of Mathematics, North Carolina State University, Raleigh, NC 27695, USA\\
\href{mailto:meyer@ncsu.edu}{meyer@ncsu.edu}
}
\maketitle              
\begin{abstract}
This paper develops the exact linear relationship between the leading eigenvector of the unnormalized modularity matrix and the eigenvectors of the adjacency matrix. We propose a method for approximating the leading eigenvector of the modularity matrix, and we derive the error of the approximation. There is also a complete proof of the equivalence between normalized adjacency clustering and normalized modularity clustering. Numerical experiments show that normalized adjacency clustering can be as twice efficient as normalized modulairty clustering.

\end{abstract}
\section{Introduction}\label{sec1}
Graph partitioning is the process of breaking a graph into smaller components so the components can be characterized by specific properties. The problem, also known as clustering or community detection, is of high interest in both academia and industry. For example, Pothen \cite{pothen1997graph} applies graph partitioning in scientific computing. Olson et al. \cite{olson2005single} uses the concept of robotics. Tolliver and Miller \cite{tolliver2006graph} discusses the possibility of using graph partitioning for image segmentation. Recently, the scientific interest in graph partitioning has centered on dividing large graphs into smaller components by matching their size. This is done by minimizing the number of edges that are cut during the process \cite{von2007tutorial}. 

A number of algorithms have been developed to solve problems related to graph partitioning. Among the many clustering methods, two spectral techniques that rely on adjacency matrices of graphs are widely used and extensively researched. Fiedler \cite{fiedler1973algebraic} develops the spectral clustering method, while Newman and Girvan \cite{newman2004finding} develop the modularity clustering method. As discussed in \cite{fiedler1973algebraic}, the eigenvalue corresponding to the second smallest eigenvector of a graph adjacency matrix is closely related to the graph's structure. It is suggested in \cite{fiedler1975property} to partition a graph based on the signs of eigenvector entries of its adjacency matrix. Newman \cite{newman2006modularity} describes modularity clustering in detail. As with Fiedler's spectral clustering method, the modularity clustering method uses entries in the eigenvector that correspond to a modularity matrix's eigenvalue. 

There are some modified versions of the spectral clustering and modularity clustering methods. Chung \cite{chung1997spectral} analyzes the properties of scaled Laplacian matrices. By utilizing normalized spectral clustering, Shi and Malik \cite{shi2000normalized} provides a method to develop normalized Laplacian matrices and use them to segment images. In \cite{ng2002spectral}, another version of normalized spectral clustering is discussed. The Laplacian matrix is scaled on one side by the researchers in their method. In \cite{bolla2011penalized}, a normalized version of modularity clustering is examined. 

Since modularity matrices are derived from adjacency matrices, it would be interesting to see if similar clustering results can be obtained from the two kinds of matrices. One main contribution of this paper is to describe the relation between clustering results using modularity matrices and adjacency matrices, and to show that using normalized modularity matrices and normalized adjacency matrices will produce the same clustering results. As a practical motivation, this paper demonstrates that clustering can be sped up by using normalized adjacency matrices rather than normalized modularity matrices.

As follows is the organization of the paper. \sref{sec1.5} contains some preliminary mathematical notations. \sref{sec2} describes how to approximate the leading eigenvector of the modularity matrix with eigenvectors of the adjacency matrix. The equivalence between normalized adjacency clustering and normalized modularity clustering is presented in \sref{sec3}. \sref{sec4} provides experimental results and discussions. \sref{sec5} contains the conclusions.

\section{Preliminaries}\label{sec1.5}
Throughout the paper, we assume $G(V,E)$ to be a connected simple graph with $m=|E|$ edges and $n=|V|$ vertices. Unless otherwise stated, $\mathbf{A}$ is assumed to represent an adjacency matrix, i.e. 
$$
\mathbf{A}_{ij} = \left\{ \begin{array}{rcl}
1 & \text{if}
& \text{nodes $i$ and $j$ are adjacent} \\ 0 & \text{if} & \text{otherwise.}
\end{array}\right.
$$

A vertex's degree is defined as
$$d_i=\sum_{i=1}^na_{ji},$$ 
and 
$$\mathbf{D}=\operatorname{diag}(d_1,d_2,\cdots,d_n)$$ 
is a degree matrix containing the degrees of the vertices in a graph. In this paper, the number of clusters is always fixed at two. Clustering methods can be applied recursively if more clusters are needed, in which case a hierarchy is built to get the desired number of clusters. It is worth noting that this approach will result in a greedy algorithm which may lead to unsatisfactory results because of poor partitioning in the first stages. 

Partitioning the graph is based on the signs of the entries in the eigenvectors. In real cases, the cases where zero entries emerge are rare, so it is assumed that there are no zero entries in the eigenvectors. Although the results are presented in this paper using adjacency matrices, it is also possible to extend the results to use similarity matrices. A graph Laplacian is defined as
\begin{equation}
\mathbf{L}=\mathbf{D}-\mathbf{A},
\end{equation}
and a modularity matrix defined as
\begin{equation}
\mathbf{M}=\mathbf{A}-\frac{\mathbf{d}\mathbf{d}^T}{2m},
\end{equation}
where
\begin{equation}\label{expressd}
\mathbf{d}=\begin{pmatrix}d_1&d_2&\cdots&d_n\end{pmatrix}^T
\end{equation}
is the vector containing the degrees of the nodes. The normalized versions of the graph Laplacian and the modularity matrix are
\begin{equation}
\mathbf{L}_{sym}=\mathbf{D}^{-\frac{1}{2}}\mathbf{L}\mathbf{D}^{-\frac{1}{2}}
\end{equation}
and
\begin{equation}
\mathbf{M}_{sym}=\mathbf{D}^{-\frac{1}{2}}\mathbf{M}\mathbf{D}^{-\frac{1}{2}},
\end{equation}
respectively. With $\mathbf{e}$ a vector that contains all 1's with proper dimension, it can be seen that $(0,\mathbf{e})$ is an eigenpair of $\mathbf{L}$ and $\mathbf{M}$, and $(0, \mathbf{D}^{\frac{1}{2}}\mathbf{e})$ is an eigenpair of $\mathbf{L}_{sym}$ and $\mathbf{M}_{sym}$.

\section{Dominant Eigenvectors of Modularity and Adjacency Matrices}\label{sec2}
As a linear combination of the eigenvectors of the corresponding adjacency matrix, the eigenvector corresponding to the largest eigenvalue of a modularity matrix is written in this section. To begin with, we state a theorem from \cite{bunch1978rank} regarding the interlacing property of a diagonal matrix and its rank-one modification, and how to calculate the eigenvectors of a diagonal plus rank one (DPR1) matrix \cite{meyer2000matrix}. The theorem is also discussed in \cite{wilkinson1965algebraic}. We will use these results in our analysis.
\begin{theorem}\label{thm1}
Let $\mathbf{P}=\mathbf{S}+\alpha\mathbf{uu}^T$, where $\mathbf{S}$ is diagonal, $\|\mathbf{u}\|_2=1$. Let $s_1\le s_2\le \cdots\le s_n$ be the eigenvalues of $\mathbf{S}$, and let $\tilde{s}_1\le \tilde{s}_2\le \cdots\le \tilde{s}_n$ be the eigenvalues of $\mathbf{P}$. Then $\tilde{s}_1\le s_1\le \tilde{s}_2\le s_2\le\cdots \le\tilde{s}_n\le s_n$ if $\alpha<0$. If the $s_i$ are distinct and all the elements of $\mathbf{u}$ are nonzero, then the eigenvalues of $\mathbf{P}$ strictly separate those of $\mathbf{S}$.
\end{theorem}
\begin{corollary}\label{thm2}
By using the notations in \thmref{thm1}, the eigenvector of $\mathbf{P}$ associated with eigenvalue $\tilde{s}_i$ can be calculated by
\begin{equation}
(\mathbf{S}-\tilde{s}_i\mathbf{I})^{-1}\mathbf{u}.
\end{equation}
\end{corollary}

By \thmref{thm1}, we know that the eigenvalues of a DPR1 matrix interlace with the eigenvalues of the original diagonal matrix. A linear combination of the eigenvectors of the corresponding adjacency matrix is then used to compute the eigenvector representing the largest eigenvalue of a modularity matrix.

According to the notation in \sref{sec1}, because $\mathbf{A}$ is an adjacency matrix, it is symmetric and is therefore orthogonally similar to a diagonal matrix. It follows that there exists an orthogonal matrix $\mathbf{U}$ and a diagonal matrix $\mathbf{\Sigma_\mathbf{A}}$ such that 
$$
\mathbf{A}=\mathbf{U}\mathbf{\Sigma_\mathbf{A}}\mathbf{U}^T.
$$
Suppose the rows and columns of $\mathbf{A}$ are ordered such that
$$
\mathbf{\Sigma_\mathbf{A}}=\operatorname{diag}(\sigma_1, \sigma_2, \cdots, \sigma_n),
$$
where $\sigma_1\ge\sigma_2\ge\cdots\ge\sigma_n$. Let $\mathbf{U}=\begin{pmatrix}\mathbf{u}_1&\mathbf{u}_2&\cdots&\mathbf{u}_n\end{pmatrix}$. Similarly, since a modularity matrix $\mathbf{M}$ is symmetric, it is orthogonally similar to a diagonal matrix. Suppose the eigenvalues of $\mathbf{M}$ are $\beta_1\ge\beta_2\ge\cdots\ge\beta_n$.

\begin{theorem}\label{thm3}
Suppose $\beta_1\ne\sigma_1$, $\beta_1\ne\sigma_2$, and $|\beta_1-\sigma_2|=\Delta$. The eigenvector corresponding to the largest eigenvalue of $\mathbf{M}$ is given by
\begin{equation}
\frac{1}{\|\mathbf{U}^T\mathbf{d}\|_2}\sum_{i=1}^n\frac{\mathbf{u}_i^T\mathbf{d}}{\sigma_i-(\sigma_2+\Delta)}\mathbf{u}_i,
\end{equation}
where $\mathbf{d}$ is defined in \eref{expressd}.
\end{theorem}
\begin{proof}
Since $\mathbf{M}=\mathbf{A}-\mathbf{d}\mathbf{d}^T/(2m)$, we have
\begin{equation}
\begin{split}
\mathbf{M} & =\mathbf{A}-\frac{\mathbf{d}\mathbf{d}^T}{2m}\\
& =\mathbf{U}\mathbf{\Sigma_\mathbf{A}}\mathbf{U}^T-\frac{\mathbf{d}\mathbf{d}^T}{2m}\\
&=\mathbf{U}(\mathbf{\Sigma_\mathbf{A}}+\rho\mathbf{y}\mathbf{y}^T)\mathbf{U}^T,
\end{split}
\end{equation}
where 
$$\mathbf{y}=\frac{\mathbf{U}^T\mathbf{d}}{\|\mathbf{U}^T\mathbf{d}\|_2}$$
and 
$$\rho=-\frac{\|\mathbf{U}^T\mathbf{d}\|_2^2}{2m}.$$ Since $\mathbf{\Sigma_\mathbf{A}}+\rho\mathbf{y}\mathbf{y}^T$ is also symmetric, it is orthogonally similar to a diagonal matrix. So we have 
$$
\mathbf{M}=\mathbf{U}\mathbf{V}\mathbf{\Sigma_\mathbf{M}}\mathbf{V}^T\mathbf{U}^T,
$$
where $\mathbf{V}$ is orthogonal and $\mathbf{\Sigma_\mathbf{M}}$ is diagonal. Since $\mathbf{\Sigma_\mathbf{A}}+\rho\mathbf{y}\mathbf{y}^T$ is a DPR1 matrix, $\rho<0$ and $\|\mathbf{y}\|_2=1$, the interlacing theorem applies to the eigenvalues of $\mathbf{A}$ and $\mathbf{M}$. More specifically, we have
\begin{equation}
\beta_n\le\sigma_n\le\beta_{n-1}\le\sigma_{n-1}\le\cdots\le\beta_2\le\sigma_2<\beta_1<\sigma_1.
\end{equation}
The strict inequalities hold because $\beta_1\ne\sigma_1$ and $\beta_1\ne\sigma_2$. Thus $|\beta_1-\sigma_2|=\Delta$ implies $\beta_1-\sigma_2=\Delta$. 
Next, let 
$$\mathbf{M}_1=\mathbf{\Sigma_\mathbf{A}}+\rho\mathbf{y}\mathbf{y}^T.$$ 
Since $\mathbf{M}=\mathbf{U}\mathbf{M}_1\mathbf{U}^T$, we have $\mathbf{M}\mathbf{U}=\mathbf{U}\mathbf{M}_1$. Suppose $(\lambda,\mathbf{v})$ is an eigenpair of $\mathbf{M}_1$, then
$$
\mathbf{M}\mathbf{U}\mathbf{v}=\mathbf{U}\mathbf{M}_1\mathbf{v}=\lambda\mathbf{U}\mathbf{v}
$$
implies that $(\lambda,\mathbf{v})$ is an eigenpair of $\mathbf{M}_1$ if and only if $(\lambda,\mathbf{Uv})$ is an eigenpair of $\mathbf{M}$. By \cref{thm2}, the eigenvector of $\mathbf{M}_1$ corresponding to $\beta_1$ is given by
\begin{equation}
\begin{split}
\mathbf{v}_1&=(\mathbf{\Sigma_{\mathbf{A}}}-\beta_1\mathbf{I})^{-1}\mathbf{y}\\
&=(\mathbf{\Sigma_{\mathbf{A}}}-(\sigma_2+\Delta)\mathbf{I})^{-1}\frac{\mathbf{U}^T\mathbf{d}}{\|\mathbf{U}^T\mathbf{d}\|_2},
\end{split}
\end{equation}
and hence the eigenvector of $\mathbf{M}$ corresponding to $\beta_1$ is given by
\begin{equation}
\begin{split}
\mathbf{m}_1&=\mathbf{Uv}_1\\
&=\mathbf{U}(\mathbf{\Sigma_{\mathbf{A}}}-(\sigma_2+\Delta)\mathbf{I})^{-1}\frac{\mathbf{U}^T\mathbf{d}}{\|\mathbf{U}^T\mathbf{d}\|_2}\\
&=\frac{1}{\|\mathbf{U}^T\mathbf{d}\|_2}\sum_{i=1}^n\frac{\mathbf{u}_i^T\mathbf{d}}{\sigma_i-(\sigma_2+\Delta)}\mathbf{u}_i.
\end{split}
\end{equation}
\end{proof}

The aim of \thmref{thm3} is to demonstrate that the vector $\mathbf{b}_1$ is a linear combination of the $\mathbf{u}_i$. Let
\begin{equation}\label{expressgammai}
\gamma_i=\frac{\mathbf{u}_i^T\mathbf{d}}{(\sigma_i-\beta_1)\|\mathbf{U}^T\mathbf{d}\|_2},
\end{equation}
the next theorem is intended to approximate $\mathbf{m}_1$, the eigenvector corresponding to the largest eigenvalue of $\mathbf{M}$, by a linear combination of $\mathbf{u}_i$ that has the largest $|\gamma_i|$, and to measure how good the approximation is by calculating the norm between $\mathbf{m}_1$ and its approximation.

\begin{theorem}\label{thm4}
With the notations and assumptions in \thmref{thm3} , and let $\gamma_i$ has the expression in \eref{expressgammai}. Suppose $i_k\in \{1,2,\cdots,n\}$, and $\gamma_{i}$ are reordered such that 
$$|\gamma_{i_n}|\le|\gamma_{i_{n-1}}|\le\cdots\le|\gamma_{i_1}|.$$ 
Then given $p\in\{1,2,\cdots,n\}$, $\mathbf{m}_1$ can be approximated by 
$$\hat{\mathbf{m}}_1 = \sum_{j=1}^p\gamma_{i_j}\mathbf{u}_{i_j},$$
with relative error
$$\frac{1}{q}\Big(\sum_{j=p+1}^n\gamma_{i_j}^2\Big)^{\frac{1}{2}},$$
where $q$ is the 2-norm of the vector $\mathbf{m}_1$.
\end{theorem}

\begin{proof}
Since 
$$\gamma_i=\frac{\mathbf{u}_i^T\mathbf{d}}{(\sigma_i-\beta_1)\|\mathbf{U}^T\mathbf{d}\|_2},$$
the vector $\mathbf{m}_1$ can be written as
$$\mathbf{m}_1=\sum_{i=1}^n\gamma_i\mathbf{u}_i=
\sum_{j=1}^n\gamma_{i_j}\mathbf{u}_{i_j}.$$
So if 
$$\hat{\mathbf{m}}_1=\sum_{j=1}^p\gamma_{i_j}\mathbf{u}_{i_j}, p\le n$$
is an approximation of $\mathbf{m}_1$, then the difference between $\mathbf{m}_1$ and its approximation is
$$\mathbf{m}_1-\hat{\mathbf{m}}_1=\sum_{j=p+1}^n\gamma_{i_j}\mathbf{u}_{i_j},$$
and the 2-norm of $\mathbf{m}_1-\hat{\mathbf{m}}_1$ is 
$$\|\mathbf{m}_1-\hat{\mathbf{m}}_1\|_2=\left\Vert\sum_{j=p+1}^n\gamma_{i_j}\mathbf{u}_{i_j}\right\Vert_2=\Big(\sum_{j=p+1}^n\gamma_{i_j}^2\Big)^{\frac{1}{2}},$$
because the $\mathbf{u}_i$ are orthonormal. So if $q$ is the 2-norm of the vector $\mathbf{m}_1$, then the relative error of the approximation is
$$\frac{\|\mathbf{m}_1-\hat{\mathbf{m}}_1\|_2}{\|\mathbf{m}_1\|}=\frac{1}{q}\Big(\sum_{j=p+1}^n\gamma_{i_j}^2\Big)^{\frac{1}{2}}.$$
\end{proof}
We can use the error provided in \thmref{thm4} to gauge the number of terms we will need to approximate the dominant eigenvector of the modularity matrix with eigenvectors of the adjacency matrix to achieve a given level of accuracy.

\section{Normalized Adjacency and Modularity Clustering}\label{sec3}
In parallel to the previous analysis, we will show that the eigenvectors corresponding to the largest eigenvalues of a normalized adjacency matrix and a normalized modularity matrix will produce the same clustering results. Bolla \cite{bolla2011penalized} mentions a similar statement without a complete proof, but Yu and Ding \cite{yu2010network} consider it from a different angle.

Suppose $\mathbf{A}$ is an adjacency matrix, and $$\mathbf{A}_{sym}=\mathbf{D}^{-\frac{1}{2}}\mathbf{AD}^{-\frac{1}{2}}$$
is the corresponding normalized adjacency matrix. Let $$\mathbf{L}=\mathbf{D}-\mathbf{A}$$
be the unnormalized Laplacian matrix and $$\mathbf{L}_{sym}=\mathbf{D}^{-\frac{1}{2}}\mathbf{L}\mathbf{D}^{-\frac{1}{2}}=\mathbf{I}-\mathbf{A}_{sym}$$
be the normalized Laplacian matrix. Finally let $\mathbf{M}$ be the unnormalized modularity matrix defined in \sref{sec1}, $$\mathbf{P}=\frac{\mathbf{d}\mathbf{d}^T}{2m},$$ 
and 
$$\mathbf{M}_{sym}=\mathbf{D}^{-\frac{1}{2}}\mathbf{M}\mathbf{D}^{-\frac{1}{2}}$$
be the normalized modularity matrix. A theorem is first stated, followed by its proof.

\begin{theorem}\label{thm6}
Suppose that zero is a simple eigenvalue of $\mathbf{M}_{sym}$, and one is a simple eigenvalue of $\mathbf{A}_{sym}$. If $\lambda\ne 0$ and $\lambda\ne 1$, then $(\lambda, \mathbf{u})$ is an eigenpair of $\mathbf{A}_{sym}$ if and only if $(\lambda, \mathbf{u})$ is an eigenpair of $\mathbf{M}_{sym}$.
\end{theorem}

This theorem may be proven by combining the following two observations. As the second observation requires more lines of explanation, we write it as a lemma.

\begin{observation}\label{obs1}
$(\lambda, \mathbf{u})$ is an eigenpair of $\mathbf{L}_{sym}$ if and only if $(1-\lambda, \mathbf{u})$ is an eigenpair of $\mathbf{A}_{sym}$ because

$$\mathbf{L}_{sym}\mathbf{u}=\lambda \mathbf{u}$$
$$\Longleftrightarrow(\mathbf{I}-\mathbf{A}_{sym})\mathbf{u}=\lambda \mathbf{u}$$
$$\Longleftrightarrow \mathbf{A}_{sym}\mathbf{u}=(1-\lambda)\mathbf{u}.$$
\end{observation}

\begin{lemma}\label{obs2}
Suppose that $0$ is a simple eigenvalue of both $\mathbf{L}_{sym}$ and $\mathbf{M}_{sym}$. It follows that if $\lambda\ne 0$ and $(\lambda, \mathbf{u})$ is an eigenpair of $\mathbf{L}_{sym}$, then $(1-\lambda, \mathbf{u})$ is an eigenpair of $\mathbf{M}_{sym}$. If $\alpha\ne 0$ and $(\alpha, \mathbf{v})$ is an eigenpair of $\mathbf{M}_{sym}$, then $(1-\alpha, \mathbf{v})$ is an eigenpair of $\mathbf{L}_{sym}$.
\end{lemma}
\begin{proof}
For $\mathbf{P}=\mathbf{d}\mathbf{d}^T/(2m)$, it is easy to observe that 
\begin{equation}
\begin{split}
\mathbf{M}_{sym}+\mathbf{L}_{sym}&=\mathbf{D}^{-\frac{1}{2}}(\mathbf{A}-\mathbf{P}+\mathbf{D}-\mathbf{A})\mathbf{D}^{-\frac{1}{2}}\\
&=\mathbf{I}-\mathbf{D}^{-\frac{1}{2}}\mathbf{P}\mathbf{D}^{-\frac{1}{2}}.  
\end{split}
\end{equation}
Let $\mathbf{E}=\mathbf{D}^{-\frac{1}{2}}\mathbf{P}\mathbf{D}^{-\frac{1}{2}}$. If $(\lambda, \mathbf{u})$ is an eigenpair of $\mathbf{L}_{sym}$, we have 
$$\lambda \mathbf{u}=\mathbf{L}_{sym}\mathbf{u}$$
$$\Longrightarrow\lambda \mathbf{u}=(\mathbf{I}-\mathbf{M}_{sym}-\mathbf{E})\mathbf{u}$$
$$\Longrightarrow(1-\lambda)\mathbf{u}=\mathbf{M}_{sym}\mathbf{u}+\mathbf{E}\mathbf{u}.$$

Note that $\mathbf{P}$ is an outer product and $\mathbf{P}\ne \mathbf{0}$, so rank($\mathbf{P}$)=1. Since $\mathbf{E}$ is congruent to $\mathbf{P}$, $\mathbf{E}$ and $\mathbf{P}$ have the same number of positive, negative and zero eigenvalues by Sylvester's law \cite{meyer2000matrix}. Therefore $$\operatorname{rank}(\mathbf{E})=\operatorname{rank}(\mathbf{P})=1.$$ 
To prove $\mathbf{E}\mathbf{u}=\mathbf{0}$, it is sufficient to prove $\mathbf{u}$ is in the nullspace of $\mathbf{E}$. Let $\mathbf{e}$ be the vector such that all its entries are one. Observe that 
\begin{equation}
\begin{split}
\mathbf{E}\cdot \mathbf{D}^{\frac{1}{2}}\mathbf{e}&=\mathbf{D}^{-\frac{1}{2}}\mathbf{P}\mathbf{D}^{-\frac{1}{2}}\mathbf{D}^{\frac{1}{2}}\mathbf{e}\\
&=\mathbf{D}^{-\frac{1}{2}}\frac{\mathbf{d}\mathbf{d}^T}{2m}\mathbf{e}\\
&=\frac{\mathbf{d}^T\mathbf{e}}{2m}(\mathbf{D}^{-\frac{1}{2}}\mathbf{d})\\
&=\mathbf{D}^{-\frac{1}{2}}\mathbf{d},
\end{split}
\end{equation}
because 
$$\mathbf{d}^T\mathbf{e}=\sum_{i=1}^nd_i=2m$$
is the sum of the degrees of all the nodes in the graph. Moreover, because $$\mathbf{D}^{-\frac{1}{2}}\mathbf{d}=\mathbf{D}^{\frac{1}{2}}\mathbf{e},$$ $(1, \mathbf{D}^{\frac{1}{2}}\mathbf{e})$ is an eigenpair of $\mathbf{E}$.
Also observe that 
\begin{equation}
\begin{split}
\mathbf{L}_{sym}\cdot \mathbf{D}^{\frac{1}{2}}\mathbf{e}&=\mathbf{D}^{-\frac{1}{2}}(\mathbf{D}-\mathbf{A})\mathbf{D}^{-\frac{1}{2}}\mathbf{D}^{\frac{1}{2}}\mathbf{e}\\
&=\mathbf{D}^{-\frac{1}{2}}\mathbf{L}\mathbf{e}=\mathbf{0}.  
\end{split}
\end{equation}
Therefore, $(0, \mathbf{D}^{\frac{1}{2}}\mathbf{e})$ is an eigenpair of $\mathbf{L}_{sym}$. Since $\mathbf{u}$ is an eigenvector of $\mathbf{L}_{sym}$ corresponding to a nonzero eigenvalue $\lambda$, we have $\mathbf{u}\perp \mathbf{D}^{\frac{1}{2}}\mathbf{e}$, so $\mathbf{u}$ is in the nullspace of $\mathbf{E}$. This gives $\mathbf{E}\mathbf{u}=\mathbf{0}$ and thus $(1-\lambda)\mathbf{u}=\mathbf{M}_{sym}\mathbf{u}$. Therefore $\lambda \mathbf{u}=\mathbf{L}_{sym}\mathbf{u}\Rightarrow(1-\lambda)\mathbf{u}=\mathbf{M}_{sym}\mathbf{u}$. 

On the other hand, if $(\alpha, \mathbf{v})$ is an eigenpair of $\mathbf{M}_{sym}$, then we have
$$\alpha\mathbf{v}=\mathbf{M}_{sym}\mathbf{v}$$
$$\Longrightarrow\alpha\mathbf{v}=(\mathbf{I}-\mathbf{L}_{sym}-\mathbf{E})\mathbf{v}$$
$$\Longrightarrow\mathbf{L}_{sym}\mathbf{v}+\mathbf{E}\mathbf{v}=(1-\alpha)\mathbf{v}.$$
Observe that 
\begin{equation}
\begin{split}
\mathbf{M}_{sym}\cdot \mathbf{D}^{\frac{1}{2}}\mathbf{e}&=\mathbf{D}^{-\frac{1}{2}}\mathbf{M}\mathbf{D}^{-\frac{1}{2}}\mathbf{D}^{\frac{1}{2}}\mathbf{e}\\
&=\mathbf{D}^{-\frac{1}{2}}\mathbf{M}\mathbf{e}=\mathbf{0}
\end{split}
\end{equation}
because the row sums of $\mathbf{M}$ are all zeros. Therefore, $(0, \mathbf{D}^{\frac{1}{2}}\mathbf{e})$ is an eigenpair of $\mathbf{M}_{sym}$. Since $\mathbf{v}$ is an eigenvector of $\mathbf{M}_{sym}$ corresponding to a nonzero eigenvalue $\alpha$, we have $\mathbf{v}\perp \mathbf{D}^{\frac{1}{2}}\mathbf{e}$, so $\mathbf{v}$ is in the nullspace of $\mathbf{E}$.
This gives $\mathbf{E}\mathbf{v}=\mathbf{0}$ and thus $(1-\alpha)\mathbf{v}=\mathbf{L}_{sym}\mathbf{v}$. Therefore $\alpha \mathbf{v}=\mathbf{M}_{sym}\mathbf{v}\Rightarrow(1-\alpha)\mathbf{v}=\mathbf{L}_{sym}\mathbf{v}$.
\end{proof} 

As a result of \thmref{thm6}, a bijection from the nonzero eigenvalues of $\mathbf{M}_{sym}$ to the nonzero eigenvalues of $\mathbf{A}_{sym}$ can be established, and the order of these eigenvalues is maintained. As zero is always an eigenvalue of $\mathbf{M}_{sym}$, the largest eigenvalue of $\mathbf{B}_{sym}$ is always nonnegative. Newman \cite{newman2006modularity} discusses when $\mathbf{B}$ can have a zero largest eigenvalue. The congruence of $\mathbf{M}_{sym}$ and $\mathbf{M}$ logically implies that if zero is the largest Eigenvalue for $\mathbf{M}$, then it is also the largest Eigenvalue for $\mathbf{B}_{sym}$. Since $(0, \mathbf{D}^{\frac{1}{2}}\mathbf{e})$ is an eigenpair of $\mathbf{M}_{sym}$ and all entries in the vector $\mathbf{D}^{\frac{1}{2}}\mathbf{e}$ are greater than zero, all nodes in the graph will be put into one cluster. We prove below that, for nontrivial cases (i.e. when the largest eigenvalue of $\mathbf{M}$ is not zero), the eigenvectors for the largest eigenvalues of both a normalized adjacency matrix and a normalized modularity matrix are the same, so in nontrivial cases they will give the same clustering results.

\begin{theorem}\label{thm7}
With the assumptions in \thmref{thm6}, and given that zero is not the largest eigenvalue of $\mathbf{M}_{sym}$, the eigenvector corresponding to the largest eigenvalue of $\mathbf{M}_{sym}$ and the eigenvector corresponding to the second largest eigenvalue of $\mathbf{A}_{sym}$ are identical.
\end{theorem}
\begin{proof}
Due to the fact that $\mathbf{L}_{sym}$ is positive semi-definite \cite{von2007tutorial}, zero is the smallest eigenvalue of $\mathbf{L}_{sym}$. Then by \oref{obs1}, one is the largest eigenvalue of $\mathbf{A}_{sym}$. Since all eigenvalues of $\mathbf{A}_{sym}$ that are not equal to one are also the eigenvalues of $\mathbf{M}_{sym}$, it follows that if the simple zero eigenvalue is not the largest eigenvalue of $\mathbf{M}_{sym}$, then the largest eigenvalue of $\mathbf{A}_{sym}$ is the second largest eigenvalue of $\mathbf{M}_{sym}$ and they have the same eigenvectors by \thmref{thm6}.
\end{proof}

Both adjacency clustering and modularity clustering involve calculation of all entries in the adjacency matrices, so they have the same time complexity of $\mathcal{O}(n^2)$. However, as shown in the next section, normalized adjacency clustering can be twice as effective as normalized modularity clustering.

\section{Experiments}\label{sec4}
In this section, synthetic and practical data sets are used to corroborate the theoretical findings presented in the previous sections. Since normalized adjacency clustering and normalized modularity clustering provides the same eigenvalues and eigenvectors, only efficiency is compared in the experiments.

\subsection{Synthetic Data Sets}
Synthetic data sets with observations from $100$ to $10,000$ are created, and for each of the data sets, the number of features is $10$. The experimental results are shown in \fref{synthetic}.

\begin{figure}[ht]
\centering
\includegraphics[width=0.9\columnwidth]{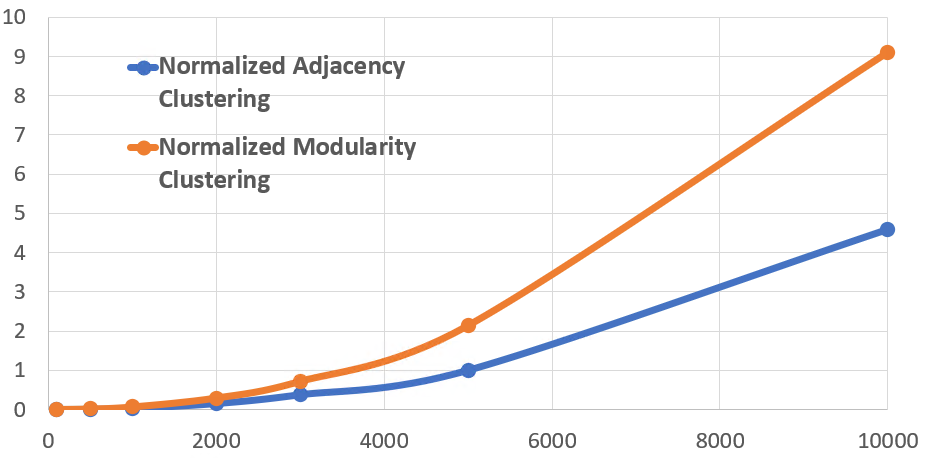}
\caption{The plot of run-time recordings of normalized adjacency clustering and normalized modularity clustering. X-axis is the number of observations, and y-axis is run-time in seconds.}
\label{synthetic}
\end{figure}

From \fref{synthetic}, it can be seen that normalized adjacency clustering (the blue line) is about twice efficient as normalized modularity clustering (the orange line).

\subsection{PenDigit Data Sets from MNIST database}
The PenDigit database is a subset of the MNIST data set \cite{lecun1998gradient,zhang2002large,hertz2004boosting,chitta2012efficient,race2014determining}. A training set of 60,000 handwritten digits from 44 writers is contained in the original data. Each data point is a row vector derived from a grayscale image. The images each have 28 pixels in height and 28 pixels in width, which makes 784 pixels in total. The row vectors contain the label of each digit as well as the lightness of each pixel. A pixel's lightness is represented by a number between 0 and 255 inclusively, with smaller numbers representing lighter pixels. The experiments were conducted using three subsets consisting of $1\&7$, $2\&3$, and $5\&6$. The experimental results are listed in \tref{tab1}.

\begin{table}[h]
\centering
\caption{The plot of run-time recordings (in seconds) of normalized adjacency clustering and normalized modularity clustering on subsets of MNIST data set}\label{tab1}
\begin{tabularx}{\textwidth}{
| >{\raggedright\arraybackslash}X 
| >{\centering\arraybackslash}X
| >{\centering\arraybackslash}X
| >{\centering\arraybackslash}X
|
}
\hline
Data &\#data points& $\mathbf{A}_{sym}$ & $\mathbf{M}_{sym}$\\
\hline
Digit1\&7 & 9085 & 4.0920 & 9.1306 \\
Digit2\&3 & 8528 & 3.5197 & 7.0120 \\
Digit5\&6 & 7932 & 3.0505 & 6.5147 \\
\hline
\end{tabularx}
\end{table}

From \tref{tab1}, it can be seen that the experimental results from real data sets are similar to the ones from synthetic data sets in that normalized adjacency clustering as around twice efficient as normalized modularity clustering.

\section{Conclusion}\label{sec5}
In this article, the exact linear relationship between the leading eigenvector of the unnormalized modularity matrix and the eigenvectors of the adjacency matrix is established. This paper demonstrates that the leading eigenvector of a modularity matrix can be written as a linear combination of the eigenvectors of an adjacency matrix, and the coefficients in the linear combination are deduced. An approximation method for the leading eigenvector of the modularity matrix is then given, along with a calculated relative error. Additionally, when the largest eigenvalue of the modularity matrix is nonzero, the normalized modularity clustering method will give the same results as using the eigenvector corresponding to the smallest eigenvalue of the normalized adjacency matrix. Experimental results indicate that using normalized adjacency clustering can be as twice efficient as normalized modularity clustering.

%
%
%
\bibliographystyle{splncs04}
\bibliography{mybib1}

\end{document}